\documentclass{article}
\usepackage[utf8]{inputenc}
\usepackage{evan}
\title{Last Iterate Convergence in Overparameterized GANs}
\author{Elbert Du}
\date{May 2021}
\usepackage{indentfirst}

\begin{document}

\maketitle

\begin{abstract}
In this work, we showed that the Implicit Update and Predictive Methods dynamics introduced in prior work satisfy last iterate convergence to a neighborhood around the optimum in overparameterized GANs, where the size of the neighborhood shrinks with the width of the neural network. This is in contrast to prior results, which only guaranteed average iterate convergence.
\end{abstract}

\section{Introduction}

One of the new and exciting results from the last few years is the Neural Tangent Kernel (NTK) \cite{jacot2020neural}, which shows that as we increase the width of every layer in a neural network to $\infty$, the function computed by the neural network approaches a linear function. This makes it possible to analyze the convergence of neural networks with only the assumption that they are sufficiently wide, and \cite{liao2020provably} proves the convergence of wide GANs using projected gradient descent in the average iterate.\\

In this paper, I adapt the analysis from \cite{liang2019interaction} to demonstrate convergence of Implicit Update (IU) dynamics and Predictive Methods (PM) to a neighborhood around the optimum, with the size of this neighborhood shrinking with the width of the neural network. To do this, I separate the error term from the known analysis in a manner similar to that of \cite{liao2020provably} and then compute the convergence rate and the size of the neighborhood by combining these two terms.

\section{Brief Overview of Liao et. al's approach}

The first key lemma from this paper (slightly abridged here, I left out some assumptions to avoid getting bogged down in the details of this lemma as this is only for motivating my approach) is one that allows us to bound the convergence using regret:

\begin{lemma}
If our sequence of functions $\{(f_t, u_t)\}_{t=1}^T$ satisfies

$$\frac{1}{T} \sum_{t=1}^T \phi(f_t,u_t) \le \min_{f} \frac{1}{T} \sum_{t=1}^T \phi(f,u_t) + \epsilon_f$$
and
$$\frac{1}{T} \sum_{t=1}^T \phi(f_t,u_t) \ge \max_{u} \frac{1}{T} \sum_{t=1}^T \phi(f_t,u) - \epsilon_u$$

Then the primal error can be bounded by $\epsilon_f + \epsilon_u$.
\end{lemma}

Given this, we now just need to compute regret for one of the players. Let $\theta_t$ be the parameters at time $t$, $\theta$ be the optimum, $\phi_t$ be the function computed by our model at time $t$, and let $\hat{\phi_t}$ be the linear approximation of $\phi_t$ given to us by NTK. Then, we can write our regret as following expression:

$$\left(\frac{1}{T} \sum_{t=1}^T \phi_t(\theta_t) - \frac{1}{T}\sum_{t=1}^T \hat{\phi_t}(\theta_t)\right) +\left( \frac{1}{T} \sum_{t=1}^T \hat{\phi_t}(\theta_t) - \frac{1}{T}\sum_{t=1}^T \hat{\phi_t}(\theta) \right) + \left( \frac{1}{T} \sum_{t=1}^T \hat{\phi_t}(\theta) - \frac{1}{T}\sum_{t=1}^T \phi_t(\theta) \right)$$

The first term is the error of the linear approximation over all of the parameters over time, and the third term is the error of the linear approximation at the optimum. These can be handled by looking at the error of the linear approximation presented in the NTK.\\

Then, the second term is simply the difference between the average iterate of our linear approximation and the optimum in the linear approximation. This can be computed using known methods in convex optimization. As such, we can see the power of separating out the error term and directly applying known results for the linear case.\\

\section{Last Iterate Results}

For our proofs, we will use the following lemma, which is the fourth statement in Lemma B.1 from \cite{liao2020provably}:
\begin{lemma}
Given a neural network with $H$ layers each of width $m = \Omega(d^{3/2}B^{-1}H^{-3/2}\log^{3/2}(m^{1/2}B^{-1}))$, where $d$ is the number of dimensions in the input, the weight vector $W$ always satisfies $||W-W_0||_2 \le B$ for $B = O(m^{1/2}H^{-6}\log^{-3}m)$ and initialization at $W_0$, then with probability at least $1-e^{-\Omega(\log^2 m)}$,

$$\norm{\nabla_W f(x,W) - \nabla_w \hat{f}(x,W)}_2 = O\left(B^{1/3}m^{-1/6}H^{5/2}\log^{1/2}m\right)$$

\end{lemma}

We will additionally assume that $C$ is invertible, and then we will prove the following simplifying lemma:

\begin{lemma}
Given a two-player game with utility function $U(\theta, \omega)$ linear in both $\theta$ and $\omega$, solving the game is equivalent to solving some game of the form

$$U'(\theta,\omega) = \theta^T C \omega$$
\end{lemma}

\begin{proof}

Since $U$ is linear, we can write

$$U(\theta,\omega) = v_1^T\theta + \theta^T C \omega + v_2^T \omega + c$$

Then, applying a translation:

$$U(\theta-\theta_0,\omega-\omega_0) = v_1^T(\theta-\theta_0) + (\theta-\theta_0)^T C (\omega-\omega_0) + v_2^T (\omega-\omega_0) + c$$

We note that $v_1^T \theta_0$, $\theta_0^T C \omega_0$, and $v_2^T\omega_0$ are all constants, so we can group them together with $c$ into a constant $c'$. Then, we have

$$U(\theta-\theta_0,\omega-\omega_0) = v_1^T\theta + \theta^T C \omega -\theta_0^T C \omega - \theta^T C \omega_0 + v_2^T \omega + c'$$

We can now set $\theta_0^TC = v_2^T$ and $C\omega = v_1$, and we end up with

$$U(\theta-\theta_0,\omega-\omega_0) = \theta^T C \omega + c'$$

and solving this game is equivalent to solving

$$U'(\theta,\omega) = \theta^T C \omega$$

so we can assume WLOG that we are optimizing a utility function

$$U(\theta,\omega) = \theta^T C \omega.$$

\end{proof}

Now, we are not actually working with linear functions. Rather, we are working in a setting where both players are represented by overparameterized neural networks. Hence, for finite widths, we instead have

$$U(\theta,\omega) = \theta^T C \omega + U'(\theta,\omega)$$

where $U'$ is a function with gradient bounded by lemma 2.\\

\subsection{Convergence of IU}

\begin{theorem}
The IU dynamics converges to a neighborhood of size $$O\left(\frac{1}{1-1/\sqrt{2}}r^{1/3}m^{-1/6}H^{5/2}\log^{1/2}m \frac{\sqrt{\lambda_{\max}(CC^T)}}{\lambda_{\min}(CC^T)}\right)$$
around the optimum.
\end{theorem}
\begin{proof}

Recall the IU dynamics:

$$\theta_{t+1} = \theta_t - \eta \nabla_\theta U(\theta_{t+1}, \omega_{t+1})$$
$$\omega_{t+1} = \omega_t - \eta \nabla_\omega U(\theta_{t+1}, \omega_{t+1})$$

We can rewrite this as:

$$\begin{bmatrix}\theta_{t+1} \\ \omega_{t+1} \end{bmatrix} = \begin{bmatrix}\theta_{t} \\ \omega_{t} \end{bmatrix} - \eta \begin{bmatrix}0 & C \\ -C^T & 0\end{bmatrix} \cdot \begin{bmatrix}\theta_{t+1} \\ \omega_{t+1} \end{bmatrix} + \eta \begin{bmatrix} \nabla U'_\theta(\theta_{t+1}, \omega_{t+1})\\ \nabla U'_\omega(\theta_{t+1},\omega_{t+1}) \end{bmatrix}$$

Now, rearranging some terms,

$$\left(I + \eta \begin{bmatrix}0 & C \\ -C^T & 0\end{bmatrix}\right)\begin{bmatrix}\theta_{t+1} \\ \omega_{t+1} \end{bmatrix} = \begin{bmatrix}\theta_{t} \\ \omega_{t} \end{bmatrix} + \eta \begin{bmatrix} \nabla U'_\theta(\theta_{t+1}, \omega_{t+1})\\ \nabla U'_\omega(\theta_{t+1},\omega_{t+1}) \end{bmatrix}$$

If we now apply lemma 2 and the triangle inequality,

$$\left(I + \eta \begin{bmatrix}0 & C \\ -C^T & 0\end{bmatrix}\right)\norm{\begin{bmatrix}\theta_{t+1} \\ \omega_{t+1} \end{bmatrix}}_2 \le \norm{\begin{bmatrix}\theta_{t} \\ \omega_{t} \end{bmatrix}}_2 + O\left(r^{1/3}m^{-1/6}H^{5/2}\log^{1/2}m \eta\right)$$

We can bound the LHS by taking the largest singular value of $K = I + \eta \begin{bmatrix}0 & C \\ -C^T & 0\end{bmatrix}$. Note that

$$KK^T = \begin{bmatrix} I +\eta^2 CC^T & 0 \\ 0 & I+\eta^2 C^TC \end{bmatrix}$$

so the singular values are at least $\sqrt{1 + \eta^2 \lambda_{\min}(CC^T)}$. If we plug in $\eta = \frac{1}{\sqrt{\lambda_{\max}(CC^T)}}$, then we can make the approximation:

$$\sigma_{\min}(K) \ge \sqrt{1 + \eta^2 \lambda_{\min}(CC^T)} \ge 1+\left(\sqrt{2}-1\right)\frac{\lambda_{\min}(CC^T)}{\lambda_{\max}(CC^T)}$$

where $\sigma_{\min}(K)$ is the least singular value of $K$.\\

Thus, if we now plug in $c$ for the constant of the big-O term,

$$\norm{\begin{bmatrix}\theta_{t+1} \\ \omega_{t+1} \end{bmatrix}}_2 \le \frac{1}{1+(\sqrt{2}-1)\lambda_{\min}(CC^T)/\lambda_{\max}(CC^T)}\norm{\begin{bmatrix}\theta_{t} \\ \omega_{t} \end{bmatrix}}_2 + cr^{1/3}m^{-1/6}H^{5/2}\log^{1/2}m \eta$$
$$\norm{\begin{bmatrix}\theta_{t+1} \\ \omega_{t+1} \end{bmatrix}}_2 \le \left(1 - \left(1-\frac{1}{\sqrt{2}}\right)\frac{\lambda_{\min}(CC^T)}{\lambda_{\max}(CC^T)}\right)\norm{\begin{bmatrix}\theta_{t} \\ \omega_{t} \end{bmatrix}}_2 + cr^{1/3}m^{-1/6}H^{5/2}\log^{1/2}m \eta$$

So the smallest distance we can guarantee we reach is $(\theta,\omega)$ such that

$$\norm{\begin{bmatrix} \theta \\ \omega \end{bmatrix}}_2\left(1-\frac{1}{\sqrt{2}}\right)\frac{\lambda_{\min}(CC^T)}{\lambda_{\max}(CC^T)} = cr^{1/3}m^{-1/6}H^{5/2}\log^{1/2}m \eta$$

or equivalently,

$$\norm{\begin{bmatrix} \theta \\ \omega \end{bmatrix}}_2 = \frac{c}{1-1/\sqrt{2}}r^{1/3}m^{-1/6}H^{5/2}\log^{1/2}m\frac{\sqrt{\lambda_{\max}(CC^T)}}{\lambda_{\min}(CC^T)}$$

and this goes to $0$ as $m \rightarrow \infty$.\\

Now, to analyze the time it takes to converge to this bound, suppose that we have

$$\norm{\begin{bmatrix} \theta_t \\ \omega_t\end{bmatrix}}_2 \ge \frac{c}{1-1/\sqrt{2}}r^{1/3}m^{-1/6}H^{5/2}\log^{1/2}m\frac{\sqrt{\lambda_{\max}(CC^T)}}{\lambda_{\min}(CC^T)} + \delta$$

Then, we can upper bound the linear approximation error term with

$$\left(1-\frac{1}{\sqrt{2}}\right)\frac{\lambda_{\min}(CC^T)}{\lambda_{\max}(CC^T)}\left(\norm{\begin{bmatrix} \theta_t \\ \omega_t\end{bmatrix}}_2-\delta\right)$$

so we get that 

$$\norm{\begin{bmatrix}\theta_{t+1} \\ \omega_{t+1} \end{bmatrix}}_2 \le \norm{\begin{bmatrix}\theta_{t} \\ \omega_{t} \end{bmatrix}}_2 - \left(1 - \frac{1}{\sqrt{2}}\right) \frac{\lambda_{\min}(CC^T)}{\lambda_{\max}(CC^T)}\delta$$

which means our distance to the neighborhood shrinks by a factor of at least $\left(1 - \left(1-\frac{1}{\sqrt{2}}\right)\frac{\lambda_{\min}(CC^T)}{\lambda_{\max}(CC^T)}\right)$, giving us the bound that after $T$ steps, we will be at distance at most 

$$\frac{c}{1-1/\sqrt{2}}r^{1/3}m^{-1/6}H^{5/2}\log^{1/2}m\frac{\sqrt{\lambda_{\max}(CC^T)}}{\lambda_{\min}(CC^T)} + \epsilon$$

where

$$T \ge \left\lceil (2+\sqrt{2}) \frac{\lambda_{\max}(CC^T)}{\lambda_{\min}(CC^T)} \log \frac{r}{\epsilon} \right\rceil$$

\end{proof}

\subsection{Convergence of PM}

\begin{theorem}
The PU dynamics converges to a neighborhood of size
$$\frac{\gamma \lambda_{\min}(CC^T)}{\lambda_{\max}(CC^T) + \gamma^2 \lambda^2_{\max}(CC^T)}\frac{(\gamma+1) \alpha}{\beta}$$
around the optimum.
\end{theorem}
\begin{proof}

The PM dynamics are defined as follows:\\

Predictive step:
$$\theta_{t+1/2} = \theta_t - \gamma \nabla_\theta U(\theta_t, \omega_t)$$
$$\omega_{t+1/2} = \omega_t - \gamma \nabla_\omega U(\theta_t, \omega_t)$$
Gradient step:
$$\theta_{t+1} = \theta_t - \eta \nabla_\theta U(\theta_{t+1/2}, \omega_{t+1/2})$$
$$\omega_{t+1} = \omega_t - \eta \nabla_\omega U(\theta_{t+1/2}, \omega_{t+1/2})$$

Furthermore, for notational convenience, we will let

$$\alpha = cr^{1/3}m^{-1/6}H^{5/2}\log^{1/2}m$$

Now, to prove the convergence, we analyze the vector of errors in the term $\eta \nabla_\theta U(\theta_{t+1/2}, \omega_{t+1/2})$ when compared to if we replaced $U$ with its linear approximation.

We let $\theta'_{t+1/2} = \theta_t - \gamma C \omega_t$ and $\omega'_{t+1/2} = \omega_t - \gamma C \theta_t$. This is the predictive step we would have gotten with the linear approximation. Similarly, we let $\theta'_{t+1}$ and $\omega'_{t+1}$ be the updates according to the linear approximation using $\theta'_{t+1/2}$ and $\omega'_{t+1/2}$.

$$\theta_{t+1} = \theta_t - \eta C \omega_{t+1/2} + \eta\nabla_\theta U'(\theta_{t+1/2}, \omega_{t+1/2}) = \theta_t - \eta C \omega'_{t+1/2} - \eta \gamma C \nabla_{\theta} U'(\theta_t,\omega_t) + \eta\nabla_\theta U'(\theta_{t+1/2}, \omega_{t+1/2})$$

and similarly for $\omega_{t+1}$.

Now, we note that this is the same as the equation we get with the exact bilinear form, except we're adding two error terms involving $\nabla_\theta U'$. As such, just as above, we can plug in $\omega'_{t+1/2} = \omega_t - \gamma C \theta_t$ and we get

$$\begin{bmatrix} \theta_{t+1} \\ \omega_{t+1} \end{bmatrix} = \left(I - \eta \begin{bmatrix} \gamma CC^T & C \\ -C^T & \gamma C^TC \end{bmatrix} \right) \cdot \begin{bmatrix} \theta_t \\ \omega_t \end{bmatrix} - \eta \gamma \begin{bmatrix} \nabla_\theta U'(\theta_t, \omega_t) \\ \nabla_\omega U'(\theta_t, \omega_t) \end{bmatrix} + \eta \begin{bmatrix} \nabla_\theta U'(\theta_{t+1/2}, \omega_{t+1/2}) \\ \nabla_\omega U'(\theta_{t+1/2}, \omega_{t+1/2}) \end{bmatrix}$$

Now, we plug in

$$ \eta = \frac{\gamma \lambda_{\min}(CC^T)}{\lambda_{\max}(CC^T) + \gamma^2 \lambda^2_{\max}(CC^T)}$$

and similarly to above, we can bound the largest singular value of $K = \left(I - \eta \begin{bmatrix} \gamma CC^T & C \\ -C^T & \gamma C^TC \end{bmatrix} \right)$. Doing this exactly as in Liang and Stokes, we get that it is upper bounded by

$$\sqrt{1 - \frac{\gamma^2 \lambda^2_{\min}(CC^T)}{\gamma^2 \lambda^2_{\max}(CC^T) + \lambda_{\max}(CC^T)}}.$$

For convenience, let this be $1-\beta$. Then, we have

$$\norm{\begin{bmatrix} \theta_{t+1} \\ \omega_{t+1} \end{bmatrix}}_2 \le (1-\beta) \cdot \norm{\begin{bmatrix} \theta_t \\ \omega_t \end{bmatrix}}_2 + \eta \gamma \alpha + \eta \alpha$$

So just like with IU, we get that the value the norm converges to is

$$\beta \norm{\begin{bmatrix} \theta \\ \omega \end{bmatrix}}_2 = \eta \gamma \alpha + \eta \alpha$$

$$\norm{\begin{bmatrix} \theta \\ \omega \end{bmatrix}}_2 =  \frac{\gamma \lambda_{\min}(CC^T)}{\lambda_{\max}(CC^T) + \gamma^2 \lambda^2_{\max}(CC^T)}\frac{(\gamma+1) \alpha}{\beta}$$

and since the error terms just give us an additive factor, we end up getting the same rate of convergence just like with IU. More specifically, after

$$T \ge 2\frac{\gamma^2 \lambda_{\max}^2(CC^T)+\lambda_{\max}(CC^T)}{\gamma^2\lambda_{\min}(CC^T)}\log \frac{r}{\epsilon}$$

iterations, we have

$$\norm{\begin{bmatrix} \theta_t \\ \omega_t \end{bmatrix}}_2 \le \frac{\gamma \lambda_{\min}(CC^T)}{\lambda_{\max}(CC^T) + \gamma^2 \lambda^2_{\max}(CC^T)}\frac{(\gamma+1) \alpha}{\beta} + \epsilon$$

and again, all of these terms are constant except for $\alpha$ which goes to $0$ as the width approaches $\infty$, so in the infinite width limit this does give us exact convergence as we would expect.
\end{proof}
\section{Future Work: Solving CO and OMD}

\subsection{Attempting to work with CO}

Recall the CO dynamics: if we define the function

$$R(\theta,\omega) = \frac{1}{2} \left(\norm{\nabla_\theta U(\theta,\omega)}^2 + \norm{\nabla_\omega U(\theta,\omega)}^2\right)$$

Then,

$$\theta_{t+1} = \theta_t - \eta\left[\nabla_\theta U(\theta_t, \omega_t) + \gamma \nabla_\theta R(\theta_t, \omega_t)\right]$$
$$\omega_{t+1} = \omega_t - \eta\left[\nabla_\omega U(\theta_t, \omega_t) + \gamma \nabla_\omega R(\theta_t, \omega_t)\right]$$

In the simple bilinear case analyzed in Liang and Stokes, this is identical to PM. However, in our case, if we expand out $R(\theta, \omega)$, we get

$$R(\theta,\omega) = \frac{1}{2} \left(\norm{C\omega + \nabla_\theta U'(\theta,\omega)}^2 + \norm{C \theta + \nabla_\omega U'(\theta,\omega)}^2\right)$$
$$ = C\omega \omega^T C^T + 2\omega^TC^T\nabla_\theta U'(\theta,\omega) + \nabla_\theta U'(\theta,\omega)^T\nabla_\theta U'(\theta,\omega) + C\theta \theta^T C^T + 2\theta^TC^T\nabla_\omega U'(\theta,\omega) + \nabla_\omega U'(\theta,\omega)^T\nabla_\omega U'(\theta,\omega)$$

So when we try to compute $\theta_{t+1}$, we get many terms with complicated second order derivatives that I was not able to compute bounds for in the time frame of this project. However, if we are able to bound the L2 norm of the gradient of each of the terms above, then we should be able to finish in the same manner as with IU and PM.

\subsection{Attempting to work with OMD}

Recall the definition of OMD:

$$\theta_{t+1} = \theta_t - 2\eta \nabla_\theta U(\theta_t, \omega_t) + \eta \nabla_\theta U(\theta_{t-1}, \omega_{t-1})$$
$$\omega_{t+1} = \omega_t - 2\eta \nabla_\omega U(\theta_t, \omega_t) + \eta \nabla_\omega U(\theta_{t-1}, \omega_{t-1})$$

In this case, we have an even larger issue. When we apply our update rule, if we define $R_1$ and $R_2$ as follows like in the bilinear case:

$$R_1 = \frac{\left(I - 2\eta \begin{bmatrix}0 & C \\ -C^T & 0 \end{bmatrix}\right) + \left(I - 4\eta^2 \begin{bmatrix} CC^T & 0 \\ 0 & C^TC \end{bmatrix}\right)^{1/2}}{2}$$
$$R_2 = \frac{\left(I - 2\eta \begin{bmatrix}0 & C \\ -C^T & 0 \end{bmatrix}\right) - \left(I - 4\eta^2 \begin{bmatrix} CC^T & 0 \\ 0 & C^TC \end{bmatrix}\right)^{1/2}}{2}$$

Now, our equation when we include the error terms looks like this:

$$\begin{bmatrix}\theta_{t+1} \\ \omega_{t+1} \end{bmatrix} = (R_1 + R_2)\begin{bmatrix} \theta_t \\ \omega_t \end{bmatrix} - R_1R_2 \begin{bmatrix}\theta_{t-1} \\ \omega_{t-1} \end{bmatrix} - 2\eta \begin{bmatrix}\nabla_\theta U'(\theta_t, \omega_t) \\ \nabla_\omega U'(\theta_t, \omega_t) \end{bmatrix} + \eta \begin{bmatrix}\nabla_\theta U'(\theta_{t-1}, \omega_{t-1}) \\ \nabla_\omega U'(\theta_{t-1}, \omega_{t-1}) \end{bmatrix}$$

If we attempt to group the terms to establish a recursion like in Liang and Stokes, we end up getting something that looks like

$$\begin{bmatrix}\theta_{t+1} \\ \omega_{t+1} \end{bmatrix} - R_2 \begin{bmatrix} \theta_t \\ \omega_t \end{bmatrix} + 2\eta \begin{bmatrix}\nabla_\theta U'(\theta_t, \omega_t) \\ \nabla_\omega U'(\theta_t, \omega_t) \end{bmatrix} = R_1\left(\begin{bmatrix} \theta_t \\ \omega_t \end{bmatrix} - R_2 \begin{bmatrix}\theta_{t-1} \\ \omega_{t-1} \end{bmatrix}  + \eta R_1^{-1} \begin{bmatrix}\nabla_\theta U'(\theta_{t-1}, \omega_{t-1}) \\ \nabla_\omega U'(\theta_{t-1}, \omega_{t-1}) \end{bmatrix}\right)$$

Unfortunately, $R_{-1}$ is most certainly not $2I$, so we cannot apply the recursion all the way down on the right hand side like we can with the bilinear case. As such, we would need to figure out a new way to analyze this if we want to prove its convergence (if it even converges).

\section{A Note on Motivating Early Stopping}

In my slides, I mistakenly believed that all of the weights varied by $O(m^{-1/2})$ from initialization, and missed the fact that the last layer can vary arbitrarily. As such, the intuition that the initialization got close to the optimum faster than the training was false.\\

Even so, there is a different idea which may motivate early stopping found in my methods here:\\

One can note that my bounds are rather loose: the vector we are adding from our error term is not necessarily in the same direction as our current state. As such, if we assume the error to be a random vector chosen uniformly from all sufficiently small vectors, we should instead expect it to sometimes even bring us closer to $0$. As such, as soon as we are within say $\alpha/10$ of the bound proven for either PM or IU here, there is a high probability (close to 1/2) that the error term will end up pushing us inside the neighborhood that we proved we can converge to. Then, we expect to essentially bounce around within this neighborhood (note that our proof can be extended to show that once we enter the neighborhood, we will never leave it). As such, as soon as we enter the neighborhood, there does not appear to be any meaning in continuing to train unless we can determine some properties about the direction of the error vector.\\

This provides some intuition as to why early stopping may be good. If we have some way to tell that at some iteration, we are very close to the optimum (for example, maybe halfway into the ball around the optimum), then stopping there would generally be better than continuing for a fixed number of iterations past there and then stopping.

\section{Other Future Work}

I did not analyze the last-iterate convergence of Simultaneous Gradient Ascent (SGA) in the overparameterized setting here. SGA does not converge in the bilinear case, but if we consider optimizing the square loss rather than just the loss, we go from a bilinear utility function to a quadratic utility function. In this setting, we are approximating a quadratic function in both the $\theta$ player and the $\omega$ player, so we have strong convexity. As such, it may be possible to show that SGA will converge in the last-iterate using squared loss.\\

It is with squared loss that Liao et al. was able to prove the average iterate convergence of projective gradient descent. As such, this would also be a promising avenue to look into for further convergence results.

\bibliographystyle{plain}
\bibliography{biblio}

\end{document}